\documentclass{article}

\usepackage{times}
\usepackage{graphicx} 
\usepackage{subfigure} 

\usepackage[square,sort,comma,numbers]{natbib}

\usepackage{algorithm}
\usepackage{algorithmic}
\usepackage{hyperref}

\usepackage{amsfonts}
\usepackage{multirow}

\usepackage{amsmath}

\newcommand{\inner}[1]{\langle #1 \rangle}
\newcommand{\reals}{\mathbb{R}}
\newcommand{\sign}{\mathrm{sign}}
\newcommand{\w}[1]{w^{(#1)}}
\DeclareMathOperator*{\E}{\mathbb{E}}
\DeclareMathOperator*{\prob}{\mathbb{P}}

\newtheorem{theorem}{Theorem}

\newtheorem{lemma}{Lemma}

\newcommand{\BlackBox}{\rule{1.5ex}{1.5ex}}  
\newenvironment{proof}{\par\noindent{\bf Proof\ }}{\hfill\BlackBox\\}

\newcommand{\lemref}[1]{Lemma~\ref{#1}}

\newcommand{\figref}[1]{Figure~\ref{#1}}
\newcommand{\tblref}[1]{Table~\ref{#1}}
\newcommand{\appref}[1]{Appendix~\ref{#1}}
\renewcommand{\eqref}[1]{Equation~\ref{#1}}

\usepackage{authblk}

\title{Decoupling ``when to update'' from ``how to update''}

\author{
  Eran Malach
  }
\author{
  Shai Shalev-Shwartz
  }
\affil{School of Computer Science, The Hebrew University, Israel}
\date{}
\begin{document} 

\maketitle

\begin{abstract} 
  Deep learning requires data. A useful approach to obtain data is to
  be creative and mine data from various sources, that were created
  for different purposes. Unfortunately, this approach often leads to
  noisy labels. In this paper, we propose a meta algorithm for
  tackling the noisy labels problem. The key idea is to decouple
  ``when to update'' from ``how to update''. We demonstrate the
  effectiveness of our algorithm by mining data for gender
  classification by combining the Labeled Faces in the Wild (LFW) 
  face recognition dataset with a
  textual genderizing service, which leads to a noisy dataset. While
  our approach is very simple to implement, it leads to
  state-of-the-art results. We analyze some convergence properties of
  the proposed algorithm.
\end{abstract} 

\section{Introduction}

In recent years, deep learning achieves state-of-the-art results in
various different tasks, however, neural networks are mostly trained
using supervised learning, where a massive amount of labeled data is
required. While collecting unlabeled data is relatively easy given the
amount of data available on the web, providing accurate labeling is
usually an expensive task. In order to overcome this problem, data
science becomes an art of extracting labels out of thin air. Some
popular approaches to labeling are crowdsourcing, where the labeling
is not done by experts, and mining available meta-data, such as text
that is linked to an image in a webpage. Unfortunately, this gives
rise to a problem of abundant noisy labels - labels may often be
corrupted \cite{ipeirotis2010quality}, which might deteriorate
the performance of neural-networks \cite{flatowrobustness}.

Let us start with an intuitive explanation as to why noisy labels are
problematic. Common neural network optimization algorithms start with
a random guess of what the classifier should be, and then iteratively
update the classifier based on stochastically sampled examples from a
given dataset, optimizing a given loss function such as the hinge loss
or the logistic loss. In this process, wrong predictions lead to an
update of the classifier that would hopefully result in better
classification performance. While at the beginning of the training
process the predictions are likely to be wrong, as the classifier
improves it will fail on less and less examples, thus making fewer and
fewer updates. On the other hand, in the presence of noisy labels, as
the classifier improves the effect of the noise increases - the
classifier may give correct predictions, but will still have to update
due to wrong labeling. Thus, in an advanced stage of the training
process the majority of the updates may actually be due to wrongly
labeled examples, and therefore will not allow the classifier to
further improve.

To tackle this problem, we propose to decouple the decision of ``when
to update'' from the decision of ``how to update''.  As mentioned
before, in the presence of noisy labels, if we update only when the
classifier's prediction differs from the available label, then at the
end of the optimization process, these few updates will probably be
mainly due to noisy labels. We would therefore like a different update
criterion, that would let us decide whether it is worthy to update the
classifier based on a given example. We would like to preserve the
behavior of performing many updates at the beginning of the training
process but only a few updates when we approach convergence.  To do
so, we suggest to train two predictors, and perform update steps only
in case of disagreement between them. This way, when the predictors
get better, the ``area'' of their disagreement gets smaller, and
updates are performed only on examples that lie in the disagreement
area, therefore preserving the desired behavior of the standard
optimization process.  On the other hand, since we do not perform an
update based on disagreement with the label (which may be due to a
problem in the label rather than a problem in the predictor), this
method keeps the effective amount of noisy labels seen throughout the
training process at a constant rate.

The idea of deciding ``when to update'' based on a disagreement
between classifiers is closely related to approaches for active
learning and selective sampling - a setup in which the learner does
not have unlimited access to labeled examples, but rather has to query
for each instance's label, provided at a given cost (see for example
\cite{settles2010active}).  Specifically, the well known
query-by-committee algorithm maintains a version space of hypotheses
and at each iteration, decides whether to query the label of a given
instance by sampling two hypotheses uniformly at random from the
version space~\cite{seung1992query,freund1997selective}.  Naturally,
maintaining the version space of deep networks seems to be
intractable. Our algorithm maintains only two deep networks. The
difference between them stems from the random
initialization. Therefore, unlike the original query-by-committee
algorithm, that samples from the version space at every iteration, we
sample from the original hypotheses class only once (at the
initialization), and from there on, we update these two hypotheses
using the backpropagation rule, when they disagree on the label. To
the best of our knowledge, this algorithm was not proposed/analyzed
previously, not in the active learning literature and especially not
as a method for dealing with noisy labels.

To show that this method indeed improves the robustness of deep
learning to noisy labels, we conduct an experiment that aims to study
a real-world scenario of acquiring noisy labels for a given dataset.
We consider the task of gender classification based on images. We did
not have a dedicated dataset for this task. Instead, we relied on the
Labeled Faces in the Wild (LFW) dataset, which contains images of
different people along with their names, but with no information about
their gender.  To find the gender for each image, we use an online
service to match a gender to a given name {(}as is suggested by
\cite{masek2015evaluation}{)}, a method which is naturally prone to
noisy labels (due to unisex names). Applying our algorithm to an
existing neural network architecture reduces the effect of the noisy
labels, achieving better results than similar available approaches,
when tested on a clean subset of the data. We also performed a
controlled experiment, in which the base algorithm is the perceptron,
and show that using our approach leads to a noise resilient algorithm,
which can handle an extremely high label noise rates of up to $40\%$. The controlled experiments are detailed in \appref{app:experiments}.

In order to provide theoretical guarantees for our meta algorithm, we
need to tackle two questions: 1. does this algorithm converge? and if
so, how quickly?  and 2. does it converge to an optimum? We give a
positive answer to the first question, when the base algorithm is the
perceptron and the noise is label flip with a constant
probability. Specifically, we prove that the expected number of
iterations required by the resulting algorithm equals (up to a
constant factor) to that of the perceptron in the noise-free
setting. As for the second question, clearly, the convergence depends
on the initialization of the two predictors. For example, if we
initialize the two predictors to be the same predictor, the algorithm
will not perform any updates. Furthermore, we derive lower bounds on
the quality of the solution even if we initialize the two predictors
at random. In particular, we show that for some distributions, the
algorithm's error will be bounded away from zero, even in the case of
linearly separable data. This raises the question of whether a better
initialization procedure may be helpful. Indeed, we show that for the
same distribution mentioned above, even if we add random label noise,
if we initialize the predictors by performing few vanilla perceptron
iterations, then the algorithm performs much better.  Despite this
worst case pessimism, we show that empirically, when working with
natural data, the algorithm converges to a good solution. We leave a
formal investigation of distribution dependent upper bounds to future
work.

\section{Related Work}


The effects of noisy labels was vastly studied in many different
learning algorithms (see for example the survey
in~\cite{frenay2014classification}), and various solutions to this
problem have been proposed, some of them with theoretically provable
bounds, including methods like statistical queries, boosting, bagging
and more
\cite{kearns1998efficient,mcdonald2003empirical,bootkrajang2012label,bootkrajang2013boosting,
  natarajan2013learning,patrini2016loss,larsen1998design,menon2016learning,awasthi2014power}.
Our focus in this paper is on the problem of noisy labels in the
context of deep learning. Recently, there have been several works
aiming at improving the resilience of deep learning to noisy labels.
To the best of our knowledge, there are four main approaches. The
first changes the loss function. The second adds a layer that tries to
mimic the noise behavior. The third groups examples into buckets.
The fourth tries to clean the data as a preprocessing step.  Beyond
these approaches, there are methods that assume a small clean data set
and another large, noisy, or even unlabeled, data set
\cite{nigam2000analyzing,blum1998combining,zhu2005semi,ando2007two}.
We now list some specific algorithms from these families.


\cite{reed2014training} proposed to change the cross entropy loss
function by adding a regularization term that takes into account the
current prediction of the network. This method is inspired by a
technique called minimum entropy regularization, detailed in
\cite{grandvalet2004semi,grandvalet2006entropy}. It was also found to
be effective by \cite{flatowrobustness}, which suggested a further
improvement of this method by effectively increasing the weight of the
regularization term during the training procedure.

\cite{mnih2012learning} suggested to use a probablilstic model that
models the conditional probability of seeing a wrong label, where the
correct label is a latent variable of the model. While
\cite{mnih2012learning} assume that the probability of label-flips
between classes is known in advance, a follow-up work by
\cite{sukhbaatar2014training} extends this method to a case were these
probabilities are unknown.  An improved method, that takes into
account the fact that some instances might be more likely to have a
wrong label, has been proposed recently in
\cite{goldberger2016training}.  In particular, they add another
softmax layer to the network, that can use the output of the last
hidden layer of the network in order to predict the probability of the
label being flipped. Unfortunately, their method involves optimizing
the biases of the additional softmax layer by first training it on a
simpler setup (without using the last hidden layer), which implies
two-phase training that further complicates the optimization process.
It is worth noting that there are some other works that suggest
methods that are very similar to
\cite{sukhbaatar2014training,goldberger2016training}, with a slightly
different objective or training method
\cite{bekker2016training,kakar2015probabilistic}, or otherwise suggest
a complicated process which involves estimation of the class-dependent
noise probabilities \cite{patrini2016making}. Another method from the
same family is the one described in \cite{xiao2015learning}, who
suggests to differentiate between ``confusing'' noise, where some
features of the example make it hard to label, or otherwise a
completely random label noise, where the mislabeling has no clear
reason.

\cite{zhuang2016attend} suggested to train the network to predict
labels on a randomly selected group of images from the same class,
instead of classifying each image individually.  In their method, a
group of images is fed as an input to the network, which merges their
inner representation in a deeper level of the network, along with an
attention model added to each image, and producing a single
prediction. Therefore, noisy labels may appear in groups with correctly
labeled examples, thus diminishing their impact. The final setup is
rather complicated, involving many hyper-parameters, rather than
providing a simple plug-and-play solution to make an existing
architecture robust to noisy labels. 

From the family of preprocessing methods, we mention
\cite{barandela2000decontamination,brodley1999identifying}, that try
to eliminate instances that are suspected to be mislabeled. Our method
shares a similar motivation of disregarding contaminated instances,
but without the cost of complicating the training process by a
preprocessing phase.

In our experiment we test the performance of our method against
methods that are as simple as training a vanilla version of neural
network. In particular, from the family of modified loss function we
chose the two variants of the regularized cross entropy loss suggested
by \cite{reed2014training} (soft and hard bootstrapping). From the
family of adding a layer that models the noise, we chose to compare to
one of the models suggested in \cite{goldberger2016training} (which is
very similar to the model proposed by \cite{sukhbaatar2014training}),
because this model does not require any assumptions or complication of
the training process.  We find that our method outperformed 
all of these competing methods, while being extremely simple to implement.

Finally, as mentioned before, our ``when to update'' rule is closely
related to approaches for active learning and selective sampling, and
in particular to the query-by-committee algorithm. In
\cite{freund1997selective} a thorough analysis is provided for various
base algorithms implementing the query-by-committee update rule, and
particularly they analyze the perceptron base algorithm under some
strong distributional assumptions.  In other works, an ensemble of
neural networks is trained in an active learning setup to improve the
generalization of neural networks
\cite{cohn1994improving,atlas1989training,krogh1995neural}. Our method
could be seen as a simplified member of ensemble methods. As mentioned
before, our motivation is very different than the active learning
scenario, since our main goal is dealing with noisy labels, rather
than trying to reduce the number of label queries. To the best of our
knowledge, the algorithm we propose was not used or analyzed in the
past for the purpose of dealing with noisy labels in deep learning.


\section{Method}
As mentioned before, to tackle the problem of noisy labels, we suggest
to change the update rule commonly used in deep learning optimization
algorithms in order to decouple the decision of ``when to update''
from ``how to update''.  In our approach, the decision of ``when to
update'' does not depend on the label. Instead, it depends on a
disagreement between two different networks.  This method could be
generally thought of as a meta-algorithm that uses two base
classifiers, performing updates according to a base learning
algorithm, but only on examples for which there is a disagreement
between the two classifiers.

To put this formally, let $\mathcal{X}$ be an instance space and $\mathcal{Y}$
be the label space, and assume we sample examples from
a distribution $\tilde{\mathcal{D}}$ over $\mathcal{X}\times\mathcal{Y}$,
with possibly noisy labels.
We wish to train a classifier $h$, coming from a hypothesis class
$\mathcal{H}$. We rely on an update rule, $U$, that updates $h$ based
on its current value as well as a mini-batch
of $b$ examples. The meta algorithm receives as input a pair of two
classifiers, $h_1,h_2\in\mathcal{H}$, the update rule, $U$, and a mini
batch size, $b$. A pseudo-code is given in Algorithm 1.

Note that we do not specify how to initialize the two base
classifiers, $h_1,h_2$.  When using deep learning as the base
algorithm, the easiest approach is maybe to perform a random
initialization.  Another approach is to first train the two
classifiers while following the regular ``when to update'' rule (which
is based on the label $y$), possibly training each classifier on a
different subset of the data, and switching to the suggested update
rule only in an advanced stage of the training process. We later show
that the second approach is preferable.

At the end of the optimization process, we can simply return one of
the trained classifiers.  If a small accurately labeled test data is
available, we can choose to return the classifier with the
better accuracy on the clean test data.

\begin{algorithm}
   \caption{Update by Disagreement}
\begin{algorithmic}
  \STATE \textbf{input}: 
\begin{ALC@g}
  \STATE an update rule $U$
  \STATE batch size $b$
  \STATE two initial predictors $h_{1},h_{2}\in\mathcal{H}$
\end{ALC@g}
   \FOR{$t=1,2,\dots,N$ }
   \STATE draw mini-batch $(x_1,y_1),\ldots,(x_b,y_b) \sim
   \tilde{\mathcal{D}}^b$
   \STATE let $S = \{(x_i,y_i) : h_1(x_i) \ne h_2(x_i)\}$
   \STATE $h_1 \leftarrow U(h_1,S)$
   \STATE $h_2 \leftarrow U(h_2,S)$
   \ENDFOR
\end{algorithmic}
\end{algorithm}

\section{Theoretical analysis}

Since a convergence analysis for deep learning is beyond our reach
even in the noise-free setting, we focus on analyzing properties of
our algorithm for linearly separable data, which is corrupted by
random label noise, and while using the perceptron as a base
algorithm.

Let $\mathcal{X} = \{ x \in \mathbb{R}^d : \|x\| \le 1\}$,
$\mathcal{Y}=\left\{ \pm1\right\} $, and let $\mathcal{D}$ be a
probability distribution over $\mathcal{X}\times\mathcal{Y}$, such
that there exists $w^*$ for which $\mathcal{D}(\{(x,y) :
y\inner{w^*,x} < 1\}) = 0$. The distribution we observe, denoted
$\tilde{\mathcal{D}}$, is a noisy version of
$\mathcal{D}$. Specifically, to sample $(x,\tilde{y}) \sim
\tilde{\mathcal{D}}$ one should sample $(x,y)\sim\mathcal{D}$ and
output $(x,y)$ with probability $1-\mu$ and $(x,-y)$ with probability
$\mu$. Here, $\mu$ is in $[0,1/2)$. 

Finally, let $\mathcal{H}$ be the class of linear classifiers, namely, 
$
\mathcal{H} = \{ x \mapsto \sign(\inner{w,x}) : w \in \reals^d \}
$.
We use the perceptron's update rule with mini-batch size of $1$. That
is, given the classifier $w_t\in \mathbb{R}^d$, the update on
example $(x_t,y_t)\in\mathcal{X}\times\mathcal{Y}$ is:
$
w_{t+1} = U(w_t,(x_t,y_t)) :=  w_t + y_t \, x_t
$.

As mentioned in the introduction, to provide a full theoretical analysis
of this algorithm, we need to account for two questions:
\begin{enumerate}
\item{does this algorithm converge? and if so, how quickly?}
\item{does it converge to an optimum?}
\end{enumerate}

Theorem 1 below provides a positive answer for the first question. 
It shows that the number of updates of our algorithm is only larger by
a constant factor (that depends on the initial vectors and the amount
of noise) relatively to the bound for the vanilla perceptron in the
noise-less case. 
\begin{theorem}
Suppose that the ``Update by Disagreement'' algorithm is run on a
sequence of random $N$ examples from $\tilde{\mathcal{D}}$, and with
initial vectors $\w{1}_0, \w{2}_0$. Denote $K = \max_i \|\w{i}_0\|$. 
Let $T$ be the number of updates performed by the ``Update by Disagreement'' algorithm. \\ 
Then, 
$
\E[T] \le \frac{3\,(4\,K +1)}{(1-2\mu)^2}\, \|w^*\|^2
$
where the expectation is w.r.t. the randomness
of sampling from $\tilde{\mathcal{D}}$.
\end{theorem}
\begin{proof}
It will be more convenient to rewrite the algorithm as follows. We
perform $N$ iterations, where at iteration $t$ we receive
$(x_t,\tilde{y}_t)$, and update
$
\w{i}_{t+1} = \w{i}_t + \tau_t\,\tilde{y}_t\,x_t ~,
$
where
\[
\tau_t = \begin{cases}
1 & ~\mathrm{if}~ \sign(\inner{\w{1}_t,x_t}) \neq \sign(\inner{\w{2}_t,x_t}) \\
0 & ~\mathrm{otherwise}
\end{cases}
\]
Observe that we can write $\tilde{y}_t = \theta_t y_t$, where
$(x_t,y_t) \sim \mathcal{D}$, and $\theta_t$ is a random variables
with $\prob[\theta_t = 1] = 1-\mu$ and $\prob[\theta_t = -1] = \mu$.
We also use the notation $v_t = y_t \inner{w^*,x_t}$ and
$\tilde{v}_t = \theta_t v_t$. Our goal is to upper bound
$\bar{T} := \E[T] = \E[\sum_t \tau_t]$.

We start with showing that
\begin{equation} \label{eqn:lemma1}
\mathbb{E}\left[\sum_{t=1}^{N}\tau_{t}\tilde{v}_{t}\right]\ge(1-2\mu)\overline{T}
\end{equation}
Indeed, since $\theta_t$ is independent of $\tau_t$ and $v_t$, we get
that:
\begin{align*}
\E[\tau_t \tilde{v}_t] &= \E[\tau_t \theta_t v_t] = \E[\theta_t] \cdot
\E[\tau_t v_t]  
= (1-2\mu) \E[\tau_t v_t] \ge (1-2\mu) \E[\tau_t] 
\end{align*}
where in the last inequality we used the fact that $v_t \ge 1$ with
probability $1$ and $\tau_t$ is non-negative. Summing over $t$ we
obtain that \eqref{eqn:lemma1} holds. 

Next, we show that for $i \in \{1,2\}$, 
\begin{equation} \label{eqn:lemma2}
\| \w{i}_t \|^{2} \le \| \w{i}_0\|^{2} +
\sum_{t=1}^{N} \tau_{t}(2 \| \w{2}_{0}-\w{1}_{0}\| +1)
\end{equation}
Indeed, since the update of $\w{1}_{t+1}$ and $\w{2}_{t+1}$ is
identical, we have that $\|\w{1}_{t+1} - \w{2}_{t+1}\| = \|\w{1}_0 -
\w{2}_0\|$ for every $t$. Now, whenever $\tau_t = 1$ we have that 
either $y_t \inner{\w{1}_{t-1},x_t} \le 0$ or $y_t
\inner{\w{2}_{t-1},x_t} \le 0$. Assume w.l.o.g. that  $y_t
\inner{\w{1}_{t-1},x_t} \le 0$. Then, 
\begin{align*}
\| \w{1}_{t}\|^2 &= \| \w{1}_{t-1}+y_{t}x_{t}\|^{2}
= \|\w{1}_{t-1}\|^2 + 2 y_t \inner{\w{1}_{t-1},x_t} + \|x_t\|^2
\le \|\w{1}_{t-1}\|^2 + 1 
\end{align*}
Second,
\begin{align*}
\| \w{2}_{t}\|^2 &= \| \w{2}_{t-1}+y_{t}x_{t}\|^{2}
= \|\w{2}_{t-1}\|^2 + 2 y_t \inner{\w{2}_{t-1},x_t} + \|x_t\|^2 \\
&\le \|\w{2}_{t-1}\|^2 + 2 y_t \inner{\w{2}_{t-1}-\w{1}_{t-1},x_t} + \|x_t\|^2 \\
&\le \|\w{2}_{t-1}\|^2 + 2 \,\|\w{2}_{t-1}-\w{1}_{t-1}\| + 1
= \|\w{2}_{t-1}\|^2 + 2 \,\|\w{2}_0-\w{1}_0\| + 1 
\end{align*}
Therefore, the above two equations imply 
$
\forall i \in \{1,2\},~
\| \w{i}_{t}\|^2 \le \|\w{i}_{t-1}\|^2 + 2 \,\|\w{2}_0-\w{1}_0\| + 1 
$.
Summing over $t$ we obtain that \eqref{eqn:lemma2} holds. 

Equipped with \eqref{eqn:lemma1} and \eqref{eqn:lemma2} we are ready to
prove the theorem. \\
Denote $K = \max_i \|\w{i}_0\|$ and note that 
$\|\w{2}_0-\w{1}_0\| \le 2 K$. 
We prove the theorem by providing upper and lower bounds on
$\E[\inner{\w{i}_t,w^*}]$. 
Combining the update rule with \eqref{eqn:lemma1} we get:
\begin{align*}
\E[\inner{\w{i}_t,w^*}] &= \inner{\w{i}_0,w^*} + \E\left[\sum_{t=1}^N \tau_t\,
                          \tilde{v}_t\right]
\ge \inner{\w{i}_0,w^*} + (1-2\mu)\bar{T}
\ge -K\,\|w^*\| + (1-2\mu)\bar{T}
\end{align*}
To construct an upper bound, first note that \eqref{eqn:lemma2}  implies that
\begin{align*}
\E[\|\w{i}_t\|^2] &\le \|\w{i}_0\|^2 + (2 \| \w{2}_{0}-\w{1}_{0}\| +1)
\bar{T}
\le K^2 + (4\,K +1) \,\bar{T}
\end{align*}
Using the above and Jensen's inequality, we get that
\begin{align*}
\E[\inner{\w{i}_t,w^*}] &\le \E[\|\w{i}_t\|\,\|w^*\|] \le
\|w^*\|\,\sqrt{\E[\|\w{i}_t\|^2]}
\le \|w^*\|\,\sqrt{  K^2 + (4\,K +1) \bar{T}}
\end{align*}
Comparing the upper and lower bounds, we obtain that
\[
-K\,\|w^*\| + (1-2\mu)\bar{T}
\le 
\|w^*\|\,\sqrt{  K^2 + (4\,K +1) \bar{T}}
\]
Using $\sqrt{a+b} \le \sqrt{a} + \sqrt{b}$, the above implies that
\[
(1-2\mu)\bar{T} - \|w^*\|\,\sqrt{(4\,K +1)}\,\sqrt{\bar{T}} -
2\,K\,\|w^*\| \le 0
\]
Denote $\alpha = \|w^*\|\,\sqrt{(4\,K +1)}$, then the above also
implies that
$
(1-2\mu)\bar{T} - \alpha\,\sqrt{\bar{T}} -
\alpha \le 0
$. \\
Denote $\beta = \alpha/(1-2\mu)$, using standard algebraic
manipulations, the above implies that
\[
\bar{T} ~\le~ \beta + \beta^2 + \beta^{1.5} \le 3\,\beta^2 ~,
\]
where we used the fact that $\|w^*\|$ must be at least $1$ for the
separability assumption to hold, hence $\beta \ge 1$. This concludes our proof. 
\end{proof}

The above theorem tells us that our algorithm converges quickly.  We
next address the second question, regarding the quality of the point
to which the algorithm converges. As mentioned in the introduction,
the convergence must depend on the initial predictors. Indeed, if
$\w{1}_0 = \w{2}_0$, then the algorithm will not make any
updates. The next question is what happens if we initialize $\w{1}_0$
and $\w{2}_0$ at random. The lemma below shows that this does not
suffice to ensure convergence to the optimum, even if the data is
linearly separable without noise. The proof for this lemma is given in \appref{app:proofs}.

\begin{lemma} \label{lem:lower_bound1} Fix some $\delta \in (0,1)$ and
  let $d$ be an integer greater than $40\,\log(1/\delta)$.  There
  exists a distribution over $\reals^d \times \{\pm 1\}$, which is
  separable by a weight vector $w^*$ for which $\|w^*\|^2 = d$, such
  that running the ``Update by Disagreement'' algorithm, with the
  perceptron as the underlying update rule, and with every coordinate
  of $\w{1}_0,\w{2}_0$ initialized according to any symmetric
  distribution over $\reals$, will yield a solution whose error is at
  least $1/8$, with probability of at least $1-\delta$.
\end{lemma}

Trying to circumvent the lower bound given in the above lemma, one may
wonder what would happen if we will initialize $\w{1}_0,\w{2}_0$
differently. Intuitively, maybe noisy labels are not such a big
problem at the beginning of the learning process. Therefore, we can
initialize $\w{1}_0,\w{2}_0$ by running the vanilla perceptron for
several iterations, and only then switch to our algorithm. Trivially,
for the distribution we constructed in the proof of
\lemref{lem:lower_bound1}, this approach will work just because in the
noise-free setting, both $\w{1}_0$ and $\w{2}_0$ will converge to
vectors that give the same predictions as $w^*$. But, what would
happen in the noisy setting, when we flip the label of every example
with probability of $\mu$? The lemma below shows that the error of the
resulting solution is likely to be order of $\mu^3$. Here again,
the proof is given in \appref{app:proofs}.

\begin{lemma} \label{lem:lower_bound2} Consider a vector
  $w^* \in \{\pm 1\}^d$ and the distribution $\tilde{\mathcal{D}}$ over
  $\reals^d \times \{\pm 1\}$ such that to sample a pair
  $(x,\tilde{y})$ we first choose $x$ uniformly at random from
  $\{e_1,\ldots,e_d\}$, set $y = \inner{w^*,e_i}$, and set $\tilde{y}
  = y$ with probability $1-\mu$ and $\tilde{y} = -y$ with probability
  $\mu$. Let $\w{1}_0,\w{2}_0$ be the result of running the vanilla
  perceptron algorithm on random examples from $\tilde{\mathcal{D}}$
  for any number of iterations. Suppose that we run the ``Update by
  Disagreement'' algorithm for an additional arbitrary number of
  iterations. Then, the error of the solution is likely to be
  $\Omega(\mu^3)$. 
\end{lemma}

To summarize, we see that without making additional assumptions on the
data distribution, it is impossible to prove convergence of our
algorithm to a good solution. In the next section we show that for
natural data distributions, our algorithm converges to a very good
solution. 

\section{Experiments}

We now demonstrate the merit of our suggested meta-algorithm using
empirical evaluation. Our main experiment is using our
algorithm with deep networks in a real-world scenario of noisy labels. In particular, we use a hypothesis
class of deep networks and a Stochastic Gradient Descent with momentum
as the basis update rule. The task is classifying face images
according to gender. As training data, we use the Labeled Faces in the
Wild (LFW) dataset for which we had a labeling of the name of the
face, but we did not have gender labeling. To construct gender labels,
we used an external service that provides gender labels based on
names. This process resulted in noisy labels. We show
that our method leads to state-of-the-art results on this task,
compared to competing noise robustness methods.
We also performed controlled experiments to demonstrate our algorithm's
performance with linear classification with varying levels of noise.
Due to the lack of space, these results are detailed in \appref{app:experiments}.

\subsection{Deep Learning}

We have applied our algorithm with a Stochastic Gradient Descent (SGD)
with momentum as the base update rule on the task of labeling images
of faces according to gender. The images were taken from the Labeled
Faces in the Wild (LFW) benchmark \cite{huang2007labeled}. This
benchmark consists of 13,233 images of 5,749 different people
collected from the web, labeled with the name of the person in the
picture. Since the gender of each subject is not provided, we follow
the method of \cite{masek2015evaluation} and use a service that
determines a person's gender by their name (if it is recognized),
along with a confidence level.  This method gives rise to ``natural''
noisy labels due to ``unisex'' names, and therefore allows us to
experiment with a real-world setup of dataset with noisy labels.

\begin{figure}[H]
\centering
\begin{tabular}{lcccc}
\hline 

Name & Kim & Morgan & Joan & Leslie\tabularnewline

Confidence & 88\% &  64\% & 82\% & 88\%\tabularnewline
\hline
\hline 
Correct & \includegraphics[width=0.1\columnwidth]{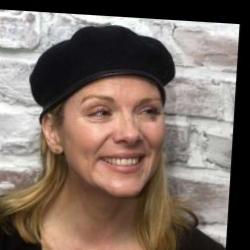} & \includegraphics[width=0.1\columnwidth]{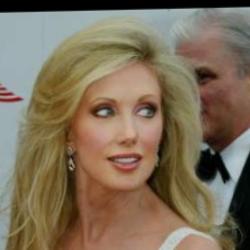} & \includegraphics[width=0.1\columnwidth]{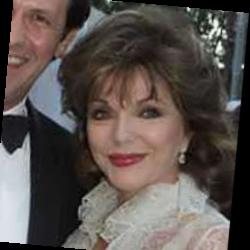} & \includegraphics[width=0.1\columnwidth]{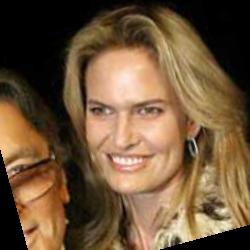}\\

Mislabeled & \includegraphics[width=0.1\columnwidth]{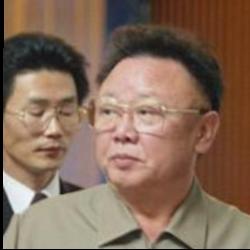} & \includegraphics[width=0.1\columnwidth]{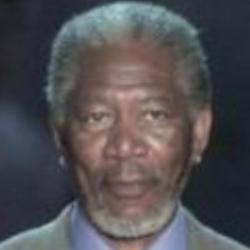} & \includegraphics[width=0.1\columnwidth]{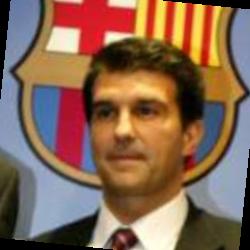} & \includegraphics[width=0.1\columnwidth]{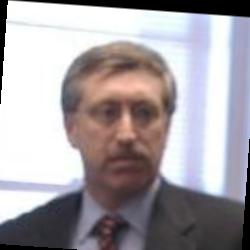}\\
\hline 
\end{tabular}%

\caption{Images from the dataset tagged as female}
\end{figure}

We have constructed train and test sets as follows. We first took all
the individuals on which the gender service gave 100\% confidence. We
divided this set at random into three subsets of equal size, denoted
$N_1,N_2,N_3$. We denote by $N_4$ the individuals on which the confidence
level is in $[90\%,100\%)$, and by $N_5$ the individuals on which the
confidence level is in $[0\%,90\%)$. Needless to say that all the sets
$N_1,\ldots,N_5$ have zero intersection with each other. 

We repeated each experiment three times, where in every time we used a
different $N_i$ as the test set, for $i \in \{1,2,3\}$. Suppose $N_1$
is the test set, then for the training set we used two configurations:
\begin{enumerate}
\item A dataset consisting of all the images that belong to names in
  $N_2,N_3,N_4,N_5$, where unrecognized names were labeled as male
  (since the majority of the subjects in the LFW dataset are males).
\item A dataset consisting of all the images that belong to names in
  $N_2,N_3,N_4$.
\end{enumerate}

We use a network architecture suggested by \cite{levi2015age}, using an available
tensorflow implementation\footnote{\url{https://github.com/dpressel/rude-carnie}.}.
It should be noted that we did not change
any parameters of the network architecture or the optimization process,
and use the default parameters in the implementation. Since the amount
of male and female subjects in the dataset is not balanced, we use
an objective of maximizing the balanced accuracy \cite{brodersen2010balanced}
 - the average accuracy obtained on either class.

 Training is done for 30,000 iterations on 128 examples mini-batch.
 In order to make the networks disagreement meaningful, we initialize
 the two networks by training both of them normally (updating on all
 the examples) until iteration \#5000, where we switch to training
 with the ``Update by Disagreement'' rule. Due to the fact that we are
 not updating on all examples, we decrease the weight of batches that
 had less than 10\% of the original examples in the original batch to
 stabilize the gradients. The exact code with the implementation
 details will be posted online.

 We inspect the balanced accuracy on our test data during the training
 process, comparing our method to a vanilla neural network training,
 as well as to soft and hard bootstrapping described in
 \cite{reed2014training} and to the s-model described in
 \cite{goldberger2016training}, all of which are using the same
 network architecture. We use the initialization parameters for
 \cite{reed2014training,goldberger2016training} that were suggested in
 the original papers. We show that while in other methods, the
 accuracy effectively decreases during the training process due to
 overfitting the noisy labels, in our method this effect is less
 substantial, allowing the network to keep improving. 

 We study two different scenarios, one in which a small clean test
 data is available for model selection, and therefore we can choose
 the iteration with best test accuracy, and a more realistic scenario
 in which there is no clean test data at hand. For the first scenario,
 we observe the balanced accuracy of the best available iteration. For
 the second scenario, we observe the balanced accuracy of the last
 iteration.

 As can be seen in \figref{fig:dnn_graphs} and the supplementary results listed in \tblref{tbl:dnn_resules} in \appref{app:experiments}, our method
 outperforms the other methods in both situations. This is true for
 both datasets, although, as expected, the improvement in performance
 is less substantial on the cleaner dataset.

 The second best algorithm is the s-model described in
 \cite{goldberger2016training}. Since our method can be applied to any
 base algorithm, we also applied our method on top of the
 s-model. This yields even better performance, especially when the
 data is less noisy, where we obtain a significant improvement. 

\begin{figure}[h]
\begin{minipage}[t]{0.5\columnwidth}%
\begin{center}
\includegraphics[width=1\columnwidth]{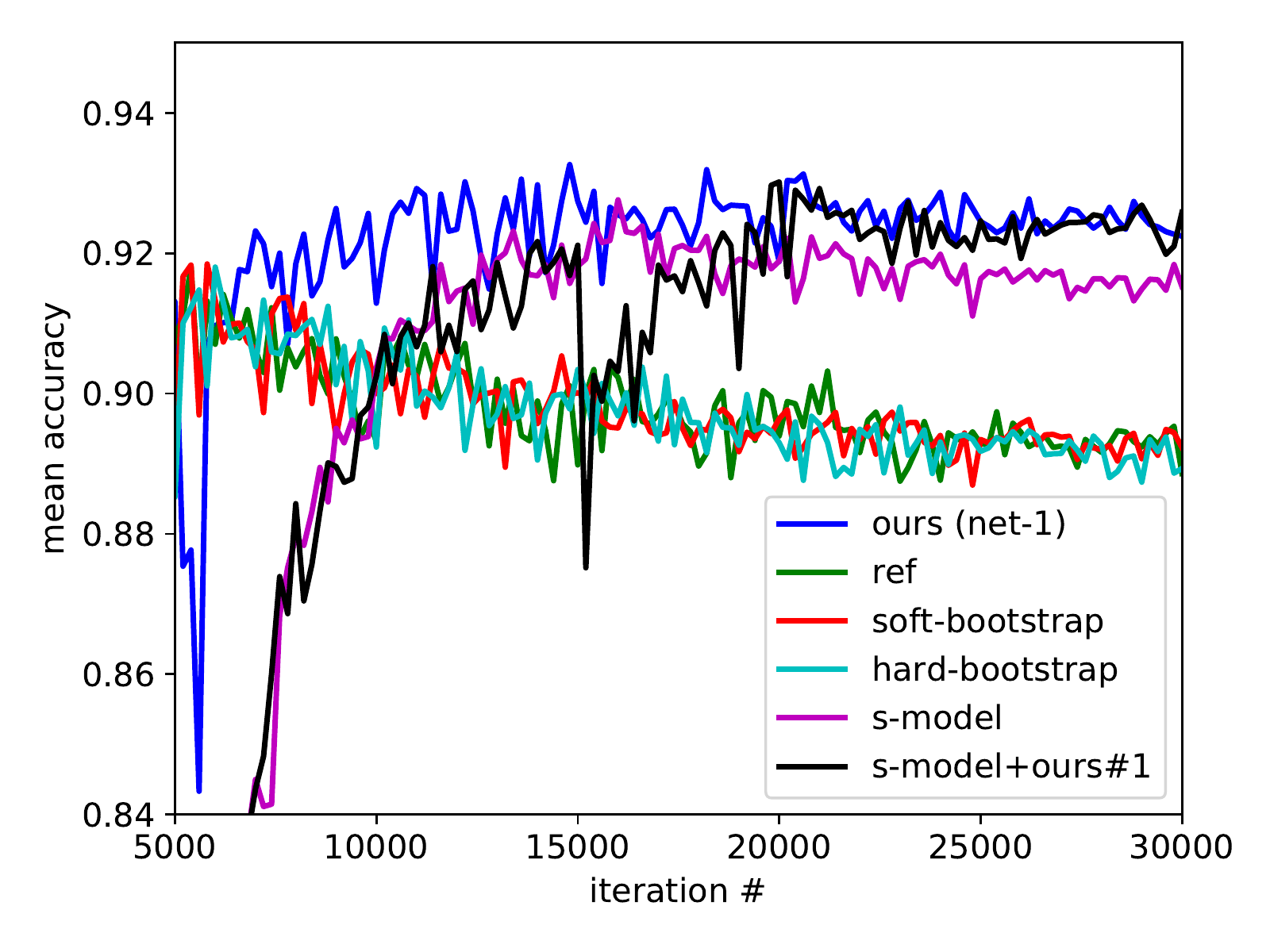}
\par
Dataset \#1 - more noise
\end{center}
\bigskip{}
\end{minipage}\hfill{}%
\begin{minipage}[t]{0.5\columnwidth}%
\begin{center}
\includegraphics[width=1\columnwidth]{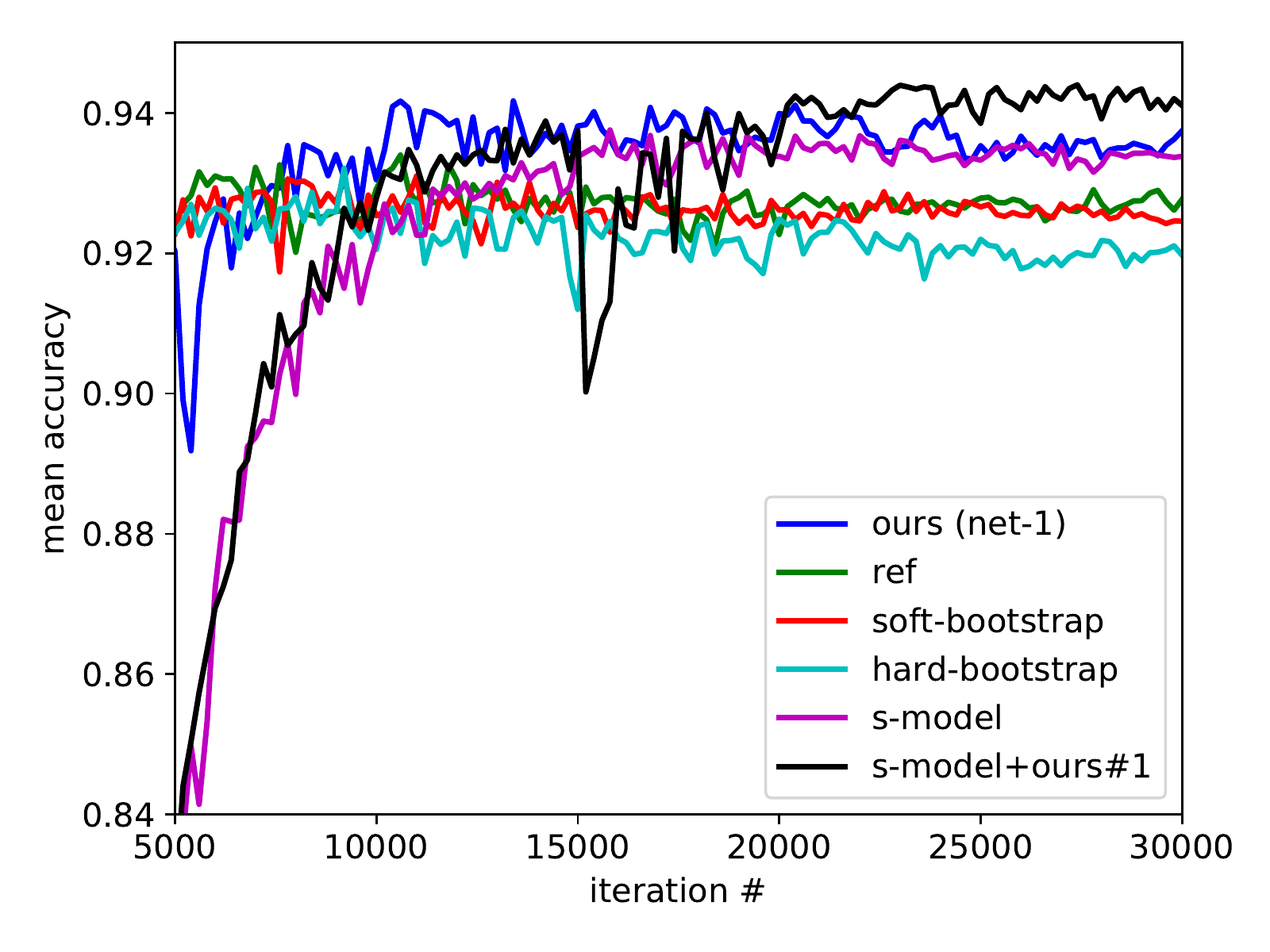}
\par
Dataset \#2 - less noise
\end{center}

\end{minipage}\hfill{}\\
\bigskip{}

\caption{Balanced accuracy of all methods on the clean test data,
when trained on the two different datasets.}

\label{fig:dnn_graphs}
\end{figure}

\section{Discussion}

We have described an extremely simple approach for supervised learning
in the presence of noisy labels. The basic idea is to decouple the
``when to update'' rule from the ``how to update'' rule. We achieve
this by maintaining two predictors, and update based on their
disagreement. We have shown that this simple approach leads to
state-of-the-art results. 

Our theoretical analysis shows that the approach leads to fast
convergence rate when the underlying update rule is the perceptron. We
have also shown that proving that the method converges to an optimal
solution must rely on distributional assumptions. There are several
immediate open questions that we leave to future work. First,
suggesting distributional assumptions that are likely to hold in
practice and proving that the algorithm converges to an optimal
solution under these assumptions. Second, extending the convergence
proof beyond linear predictors. While obtaining absolute convergence
guarantees seems beyond reach at the moment, coming up with oracle
based convergence guarantees may be feasible.

\paragraph{Acknowledgements:} This research is supported by the European Research Council (TheoryDL project).

\clearpage

\bibliography{labelnoise}
\bibliographystyle{plain}

\clearpage

\appendix

\begin{center}
{\LARGE
Supplementary Material
}
\end{center}

\section{Proofs}

\label{app:proofs}
\begin{proof} of \lemref{lem:lower_bound1}: \\
  Let the distribution over instances be concentrated uniformly over
  the vectors of the standard basis, $e_1,\ldots,e_d$. Let $w^*$ be
  any vector in $\{\pm 1\}^d$. Fix some $i$. Then, with probability
  $1/4$ over the choice of $\w{1}_0,\w{2}_0$, we have that the signs
  of $\inner{\w{1}_0,e_i}, \inner{\w{2}_0,e_i}$ agree with each other,
  but disagree with $\inner{w^*,e_i}$. It is easy to see that the
  $i$'th coordinate of $\w{1}$ and $\w{2}$ will never be
  updated. Therefore, no matter how many iterations we will perform,
  the solution will be wrong on $e_i$. It follows that the probability
  of error is lower bounded by the random variable $\frac{1}{d} \sum_{i=1}^d Z_i$,
  the $Z_i$ are i.i.d. Bernoulli variables with $\prob[Z_i=1] =
  1/4$. Using Chernoff's inequality, 
\[
\prob\left[ \frac{1}{d} \sum_{i=1}^d Z_i < 1/8 \right] ~\le~ \exp(-d
C) ~,
\]
where $C = \frac{3}{112}$. It follows that if $d \ge \log(1/\delta)/C$
then with probability of at least $1-\delta$ we will have that the
error of the solution is at least $1/8$. 
\end{proof}

\begin{proof} of \lemref{lem:lower_bound2}: \\
  Let $w_t$ be a random vector indicating the vector of the perceptron
  after $t$ iterations. Fix some $i$ and w.l.o.g. assume that
  $w^*_i = 1$. The value of $w_t$ at the $i$'th coordinate is always
  in the set $\{-1,0,1\}$. Furthermore, it alters its value like a
  Markov chain with a transition matrix of
\[
P = \begin{pmatrix}
\mu & 1-\mu & 0 \\
\mu & 0  & 1-\mu \\
0 & \mu & 1-\mu
\end{pmatrix}
\]
It is easy to verify that the stationary distribution over
$\{-1,0,1\}$ is 
\[
\pi = \left(\frac{\mu^2}{\mu + (1-\mu)^2} , \frac{\mu(1-\mu)}{\mu
    +(1-\mu)^2}, \frac{ (1-\mu)^2}{\mu + (1-\mu)^2}\right) ~.
\]
Now, the probability that our algorithm will fail on the $i$'th
coordinate is lower bounded by the probability that the $i$'th
coordinate of both $\w{1},\w{2}$ will be $0$ and then our algorithm
will see a flipped label. This would happen with probability of order
of $\mu^3$ for a small $\mu$.
\end{proof}

\clearpage

\section{Experimental Results}

\label{app:experiments}

We show our algorithm's performance in two controlled setups, 
using a perceptron based algorithm.
In the first setup we test we run our algorithm
on synthetic data that is generated by
randomly sampling instances from the unit ball in $\mathbb{R}^d$, with
different probabilities for random label-flips.
In the second setup we test our performance on a binary classification
task based on the MNIST dataset, again with random label-flips with
different probabilities.
We show that in both scenarios, our
adaptation of the perceptron algorithm results in resilience for large
noise probabilities, unlike the vanilla perceptron algorithm which
fails to converge on even small amounts of noise.

\subsection{Linear Classification on Synthetic Data}
To test the performance of the suggested perceptron-like algorithm,
we use synthetic data in various dimensions, generated in the following
process:
\begin{enumerate}
\item Randomly choose $w^{*}\in\mathbb{R}^{d}$ with a given norm $\| w^{*}\|=10^{3}$
\begin{enumerate}
\item In each iteration, draw vectors $x\in\mathbb{R}^{d}$ from the uniform distribution
on the unit ball until $|\inner{w^*,x}| \ge 1$, and then set $y =
\sign(\inner{w^*,x})$. 
\item With probability $\mu<0.5$, flip the sign of $y$.
\end{enumerate}
\end{enumerate}
The above was performed for different values of $\mu$, and
repeated 5 times for each setup. In \figref{fig:synthetic} we depict
the average performance over the 5 runs. As can be seen, our algorithm
greatly improves the noise resilience of the vanilla perceptron. 

\begin{figure}[H]
\begin{minipage}[t]{0.48\columnwidth}%
\begin{center}
\includegraphics[width=1\columnwidth]{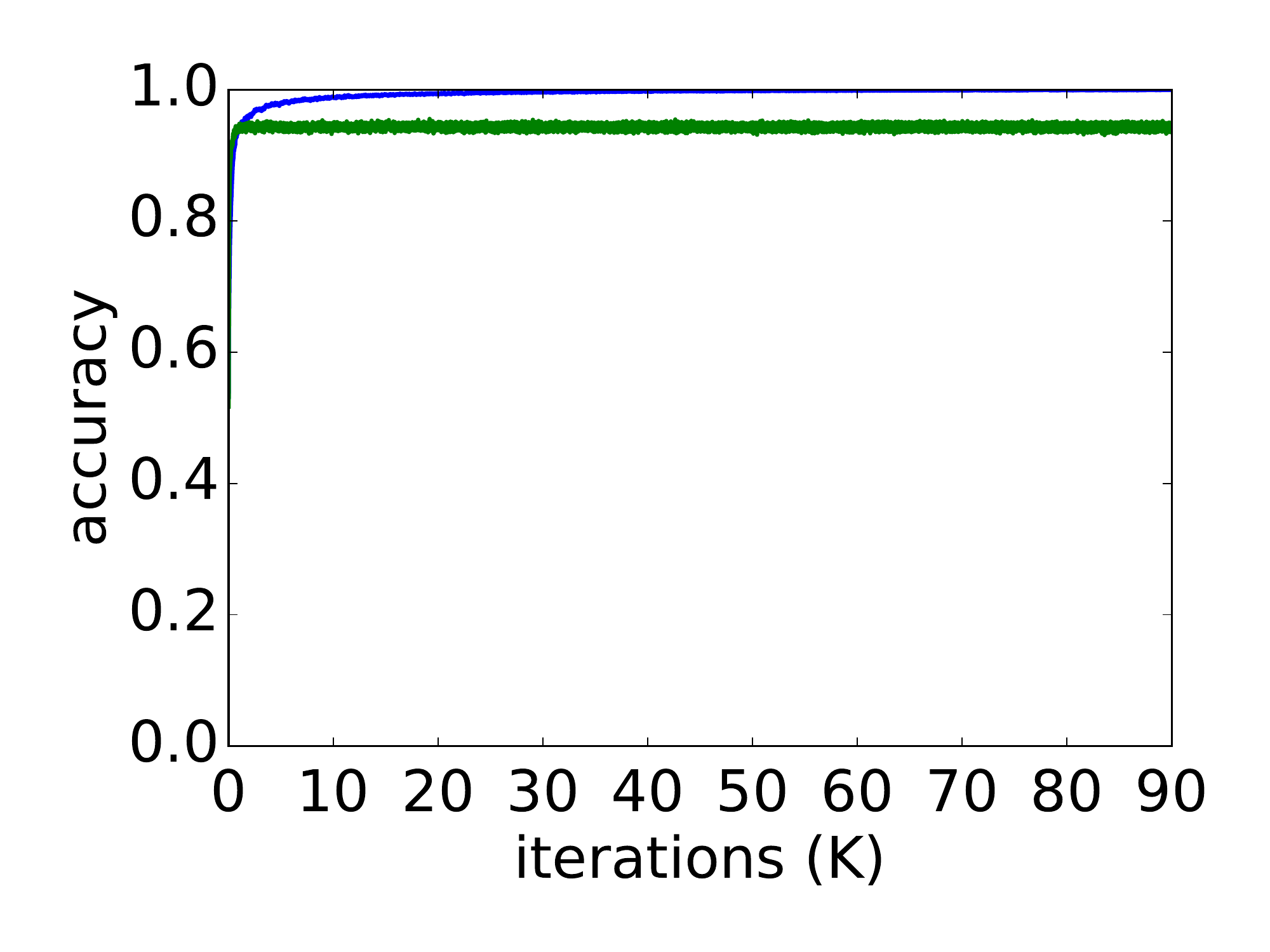}
\par\end{center}

\begin{center}
$d=100,\;\mu=0.01$
\par\end{center}%
\end{minipage}\hfill{}%
\begin{minipage}[t]{0.48\columnwidth}%
\begin{center}
\includegraphics[width=1\columnwidth]{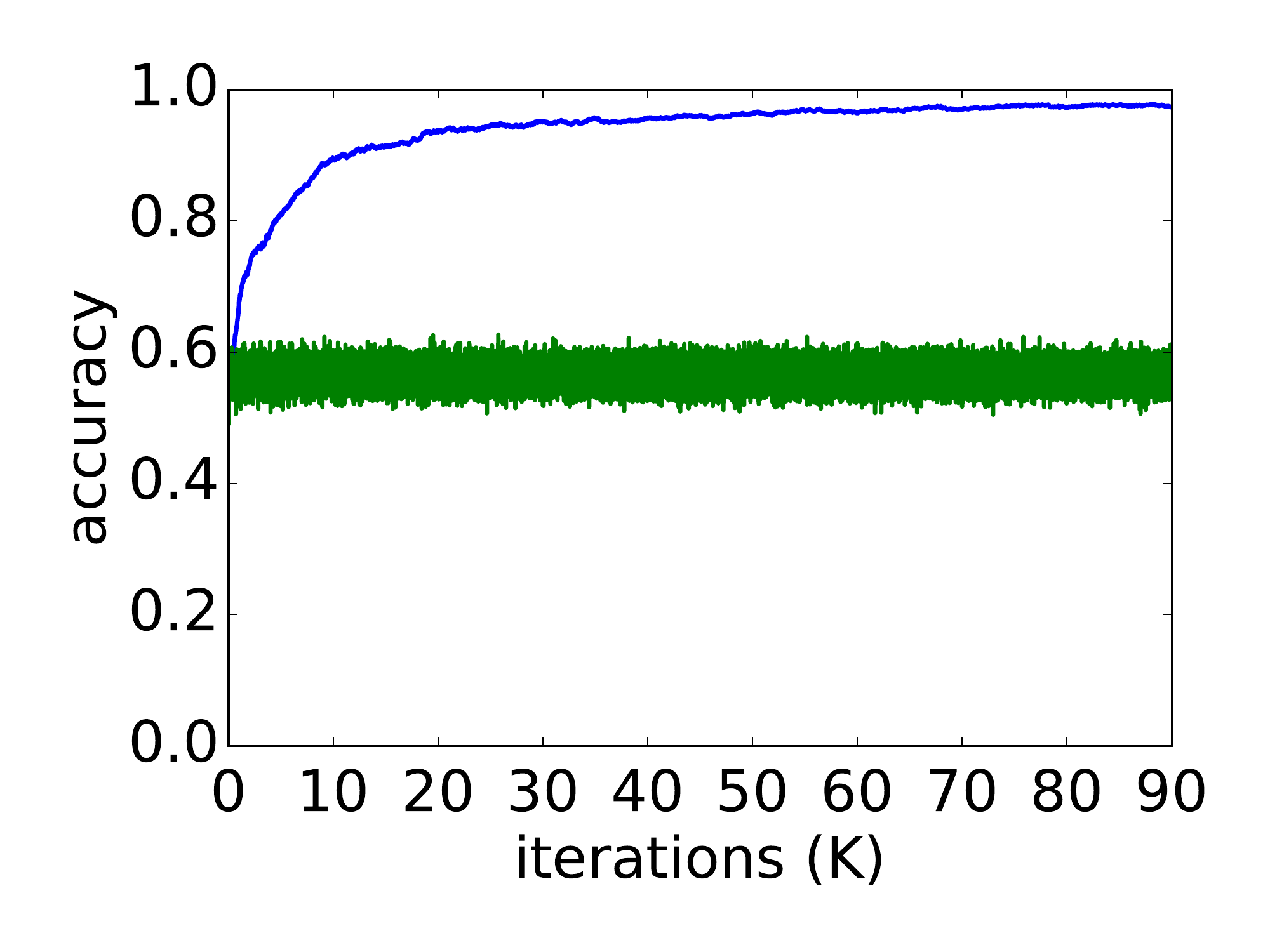}
\par\end{center}

\begin{center}
$d=100,\;\mu=0.4$
\par\end{center}%
\end{minipage}

\caption{Mean accuracy of our algorithm (blue line) compared to a vanilla perceptron
update rule (green line), averaged across 5 randomly initialized training
sessions, testing different noise rate values. Each
iteration is tested against a test set of 10K correctly labeled examples.} \label{fig:synthetic}
\end{figure}

\clearpage

\subsection{Linear Classification on MNIST Data Noisy Labels}

Here we use a binary classification task of discriminating between
the digits 4 and 7, from the MNIST dataset.

We tested the performance of the above algorithm against the regular
perceptron algorithm with various levels of noise.\\
\begin{figure}[H]
\begin{minipage}[t]{0.48\columnwidth}%
\begin{center}
\includegraphics[width=1\columnwidth]{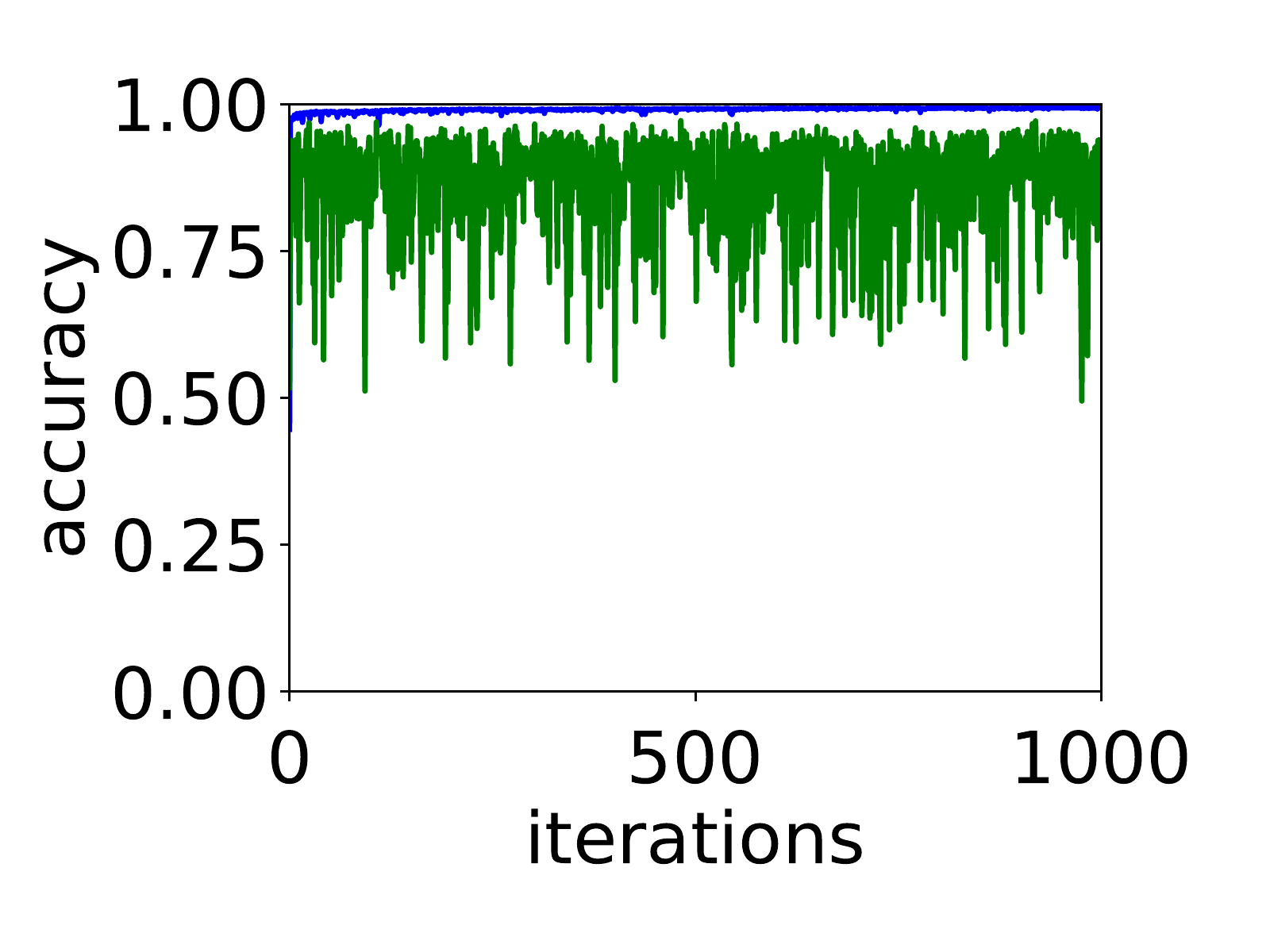}
\par\end{center}

\begin{center}
$\mu=0.1$
\par\end{center}%
\end{minipage}\hfill{}%
\begin{minipage}[t]{0.48\columnwidth}%
\begin{center}
\includegraphics[width=1\columnwidth]{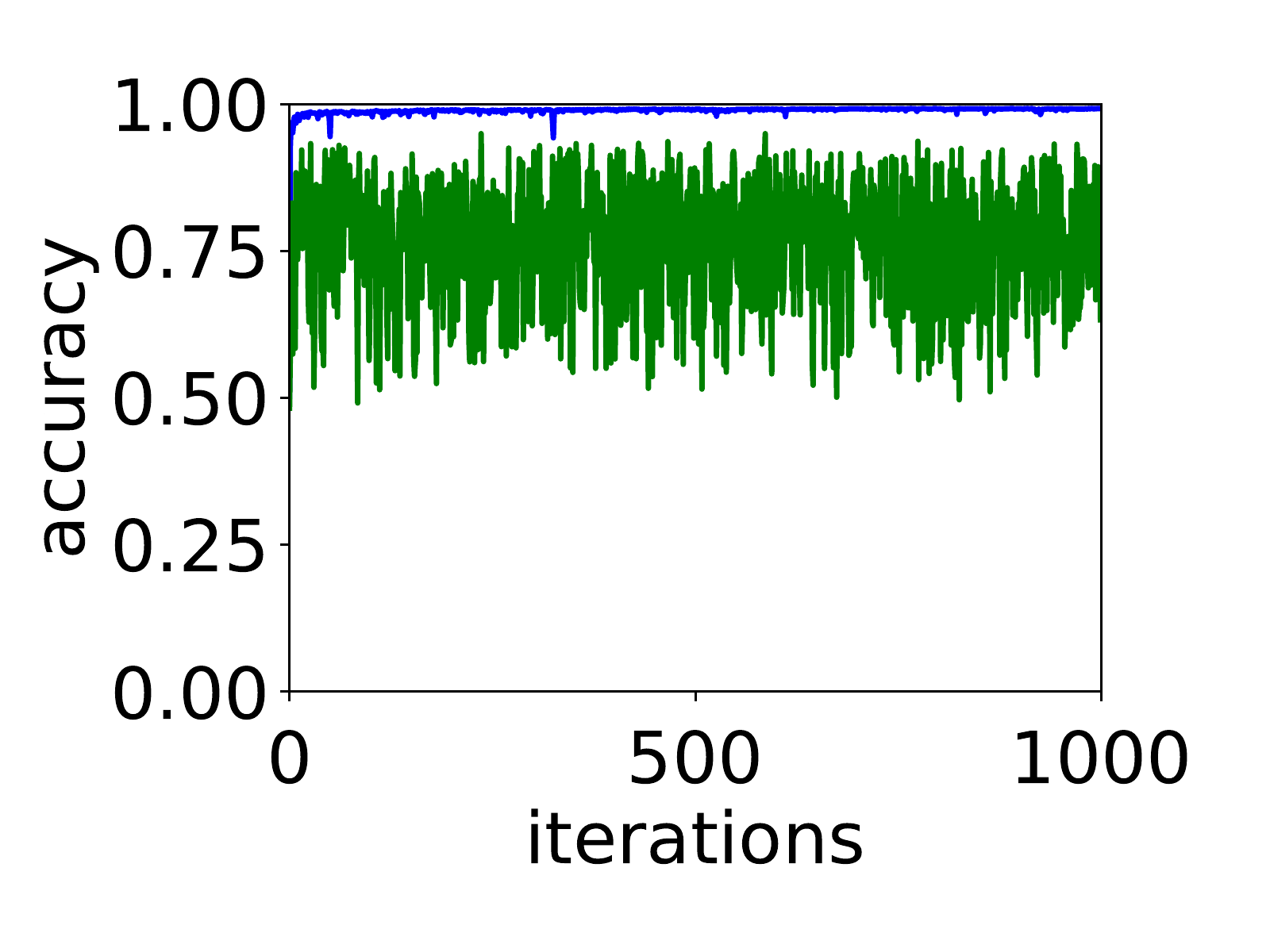}
\par\end{center}

\begin{center}
$\mu=0.2$
\par\end{center}%
\end{minipage}\hfill{}%

\bigskip{}

\begin{minipage}[t]{0.48\columnwidth}%
\begin{center}
\includegraphics[width=1\columnwidth]{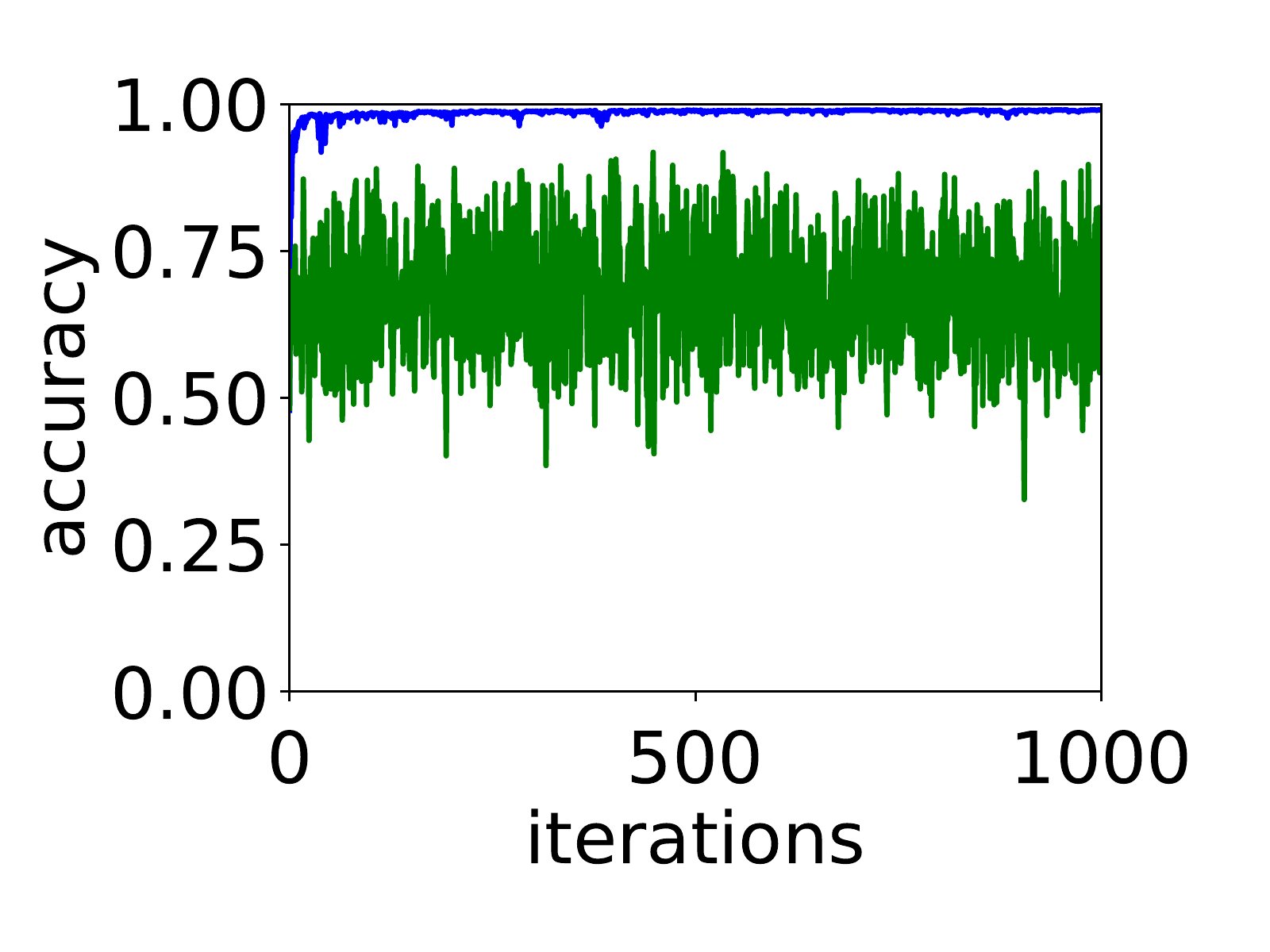}
\par\end{center}

\begin{center}
$\mu=0.3$
\par\end{center}%
\end{minipage}\hfill{}%
\begin{minipage}[t]{0.48\columnwidth}%
\begin{center}
\includegraphics[width=1\columnwidth]{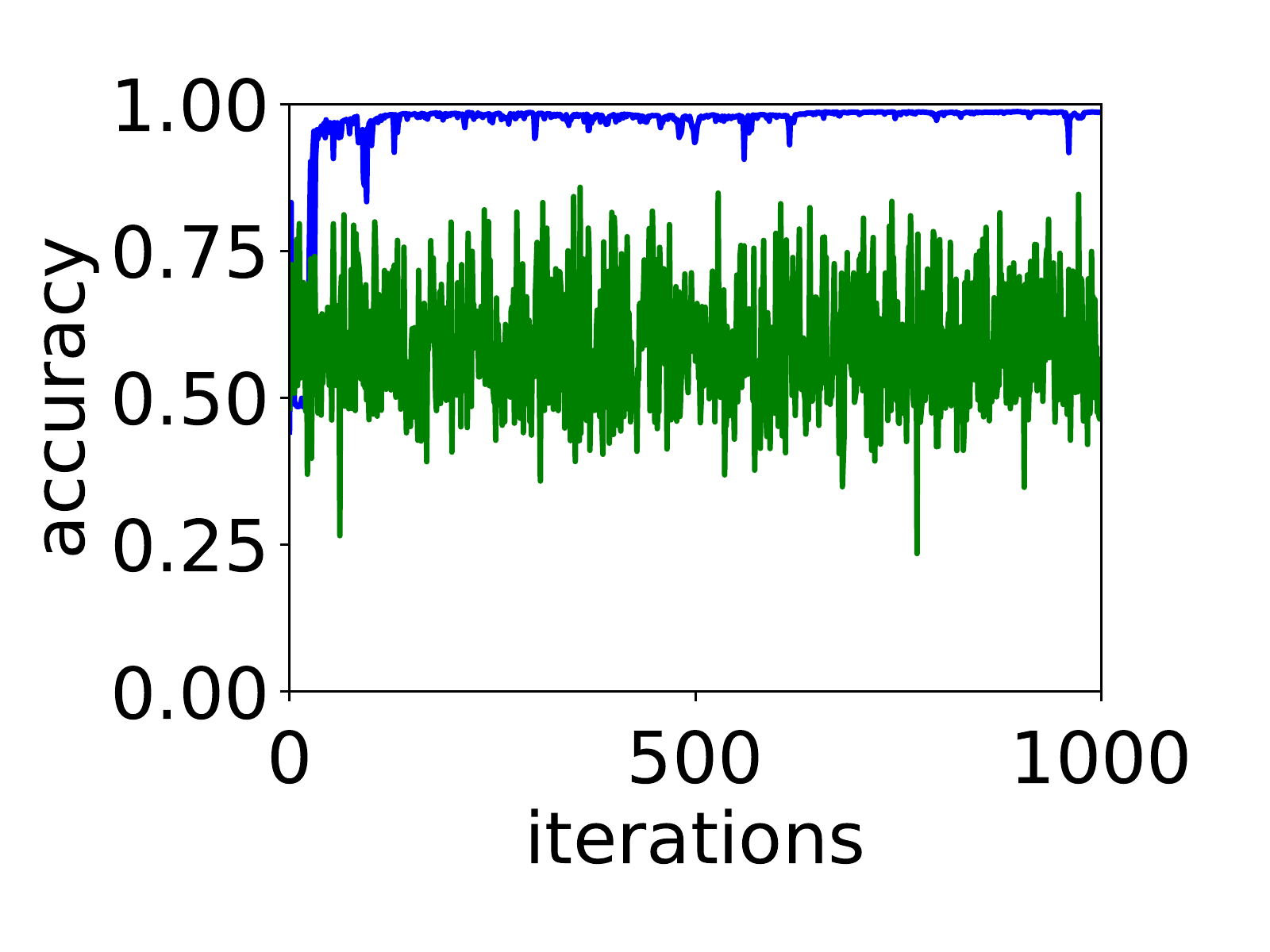}
\par\end{center}

\begin{center}
$\mu=0.4$
\par\end{center}%
\end{minipage}\hfill{}%

\bigskip{}

\begin{minipage}[t]{1\columnwidth}%
\begin{tabular}{lcccccc}
\hline 
$\mu=$ & $0.1$ & $0.2$ & $0.3$ & $0.4$\tabularnewline
\hline 
\hline 
best acc. (ours) & \textbf{99.4} & \textbf{99.2} & \textbf{99.0} & \textbf{98.7}\tabularnewline
\hline 
best acc. (perceptron) & 97.2 & 95.0 & 91.8 & 85.8\tabularnewline
\hline 
mean last 100 iters (ours) & \textbf{99.3} & \textbf{99.1} & \textbf{98.9} & \textbf{98.4}\tabularnewline
\hline 
mean last 100 iters (perc.) & 87.0 & 77.7 & 65.4 & 59.2\tabularnewline
\hline 
\end{tabular}%
\end{minipage}

\caption{Mean accuracy of our algorithm (blue line) compared to a regular perceptron
update rule (green line), with different noise rates. In all training
sessions we performed 1M iterations, randomly drawing examples from
the MNIST train set, and testing the accuracy of both algorithms on
the MNIST test set every 1000 iterations.}
\end{figure}

\clearpage

\subsection{Deep Learning Detailed Results}
The table below details the results of the LFW experiment,
showing the balanced accuracy of all the different methods
for dealing with noisy labels. We show the results on the
best iteration and on the last iteration. We observe that
our method outperforms other alternative, and combining it
with the s-model of \cite{goldberger2016training} results
in an even better improvement.
\begin{table}[h]
\caption{Accuracy on the test data
in the best iteration (with respect to the test data) and the last
iteration, achieved by each method during the training process.}
\vskip 0.15in
\begin{center}
\begin{tabular}{lcccccc}
\hline 
\multirow{2}{*}{Dataset \#1} & \multicolumn{3}{c}{Accuracy (best iteration)} & \multicolumn{3}{c}{Accuracy (last iteration)} \tabularnewline
 & Male & Female & \textbf{Mean} & Male & Female & \textbf{Mean}\tabularnewline
\hline 
\hline 
ours (net \#1) & 94.4 $\pm$ 0.7 & 92.7 $\pm$ 0.2 & \textbf{93.6 $\pm$ 0.2} & 94.8 $\pm$ 0.8 & 89.7 $\pm$ 1.3 & \textbf{92.2 $\pm$ 0.6}\tabularnewline
ours (net \#2) & 93.5 $\pm$ 1.1 & 93.2 $\pm$ 0.6 & 93.4 $\pm$ 0.3 & 93.7 $\pm$ 0.8 & 90.1 $\pm$ 0.9 & 91.9 $\pm$ 0.4\tabularnewline
s-model+ours \#1 & 93.3 $\pm$ 1.7 & 93.8 $\pm$ 1.4 & \textbf{93.6 $\pm$ 0.4} & 93.7 $\pm$ 1.1 & 91.4 $\pm$ 1.0 & \textbf{92.6 $\pm$ 0.1}\tabularnewline
s-model+ours \#2 & 94.2 $\pm$ 0.7 & 91.7 $\pm$ 0.6 & 93.0 $\pm$ 0.2 & 93.6 $\pm$ 1.3 & 91.6 $\pm$ 1.5 & 92.6 $\pm$ 0.1\tabularnewline
\hline 
\hline 
baseline & 91.6 $\pm$ 2.2 & 92.7 $\pm$ 1.8 & 92.2 $\pm$ 0.2 & 94.5 $\pm$ 0.7 & 83.3 $\pm$ 3.2 & 88.9 $\pm$ 1.3\tabularnewline
bootstrap-soft & 92.5 $\pm$ 0.6 & 91.9 $\pm$ 0.6 & 92.2 $\pm$ 0.2 & 94.5 $\pm$ 0.7 & 84.0 $\pm$ 1.7 & 89.2 $\pm$ 0.8\tabularnewline
bootstrap-hard & 92.4 $\pm$ 0.7 & 91.9 $\pm$ 1.0 & 92.1 $\pm$ 0.3 & 94.7 $\pm$ 0.2 & 83.2 $\pm$ 1.7 & 88.9 $\pm$ 0.8\tabularnewline
s-model & 94.5 $\pm$ 0.7 & 91.3 $\pm$ 0.4 & 92.9 $\pm$ 0.5 & 93.3 $\pm$ 2.0 & 89.8 $\pm$ 1.3 & 91.5 $\pm$ 0.4\tabularnewline
\hline 
\end{tabular}%
%
%

\vskip 0.2in

\begin{tabular}{lcccccc}
\hline 
\multirow{2}{*}{Dataset \#2}  & \multicolumn{3}{c}{Accuracy (best iteration)} & \multicolumn{3}{c}{Accuracy (last iteration)}\tabularnewline
 & Male & Female & \textbf{Mean} & Male & Female & \textbf{Mean} \tabularnewline
\hline 
\hline 
ours (net \#1) & 95.5 $\pm$ 0.8 & 93.6 $\pm$ 0.9 & 94.5 $\pm$ 0.2 & 95.4 \textbf{$\pm$} 1.1 & 92.1 \textbf{$\pm$} 0.7 & 93.7 $\pm$ 0.2\tabularnewline
ours (net \#2) & 95.7 $\pm$ 1.5 & 93.0 $\pm$ 1.8 & 94.4 $\pm$ 0.2 & 95.9 \textbf{$\pm$} 0.6 & 91.6 \textbf{$\pm$} 0.6 & 93.7 $\pm$ 0.3\tabularnewline
s-model+ours \#1 & 95.5 $\pm$ 0.5 & 94.0 $\pm$ 0.7 & \textbf{94.8 $\pm$ 0.2} & 95.3 $\pm$ 1.3 & 92.9 $\pm$ 2.2 & \textbf{94.1 $\pm$ 0.4}\tabularnewline
s-model+ours \#2 & 95.1 $\pm$ 0.8 & 93.9 $\pm$ 1.5 & 94.5 $\pm$ 0.3 & 95.6 $\pm$ 1.2 & 92.5 $\pm$ 1.7 & 94.0 $\pm$ 0.2\tabularnewline
\hline 
\hline 
baseline & 93.6 $\pm$ 0.7 & 93.9 $\pm$ 0.8 & 93.8 $\pm$ 0.3 & 96.2 $\pm$ 0.2 & 89.4 $\pm$ 1.6 & 92.8 $\pm$ 0.8\tabularnewline
bootstrap-soft & 94.8 $\pm$ 1.0 & 92.2 $\pm$ 0.6 & 93.5 $\pm$ 0.4 & 96.2 $\pm$ 0.6 & 88.7 $\pm$ 2.0 & 92.5 $\pm$ 0.7\tabularnewline
bootstrap-hard & 93.9 $\pm$ 1.2 & 92.8 $\pm$ 0.7 & 93.4 $\pm$ 0.4 & 96.1 $\pm$ 0.3 & 87.9 $\pm$ 1.6 & 92.0 $\pm$ 0.6\tabularnewline
s-model & 94.8 $\pm$ 1.0 & 93.3 $\pm$ 0.4 & 94.1 $\pm$ 0.3 & 94.5 $\pm$ 0.6 & 92.3 $\pm$ 0.2 & 93.4 $\pm$ 0.4\tabularnewline
\hline 
\end{tabular}%
%
%

\end{center}
\label{tbl:dnn_resules}
\end{table}

\end{document}